\documentclass[runningheads]{llncs}
\usepackage{times}
\usepackage{helvet}
\usepackage{courier}
\usepackage{url}
\usepackage{graphicx}
\usepackage{amssymb}
\usepackage{amsmath}
\usepackage{xspace}
\usepackage{lineno}
\usepackage{tabularx}
\usepackage{mathptmx}
\usepackage{microtype}
\usepackage{verbatim}
\usepackage{xcolor}
\usepackage{enumitem}
\usepackage{lineno}
\usepackage{todonotes}
\usepackage[normalem]{ulem}
\setlist{nosep}
\AtBeginDocument{
\setlength{\abovedisplayskip}{3pt plus 0.5pt minus 0.5pt}
\setlength{\belowdisplayskip}{3pt plus 0.5pt minus 0.5pt}
\setlength{\abovedisplayshortskip}{1pt plus 0.5pt minus 0.5pt}
\setlength{\belowdisplayshortskip}{3pt plus 0.5pt minus 0.5pt}
}

\newcommand{\To}{\Rightarrow}

\newcommand{\PROP}{\ensuremath{\mathrm{PROP}}\xspace}
\newcommand{\LIT}{\ensuremath{\mathrm{Lit}}\xspace}

\newcommand{\FACTS}{\ensuremath{F}}
\newcommand{\LAB}{\ensuremath{\mathrm{Lab}}\xspace}

\newcommand{\non}{\ensuremath{\mathord{\sim}}}

\RequirePackage{ifthen}
\newcommand{\set}[2][\relax]{%
 \ifthenelse{\equal{#1}{auto}}{\left\{#2\right\}}{#1\{#2#1\}}%
}

\title{Resource-driven Substructural Defeasible Logic}

\author{
 Francesco Olivieri$^1$,
 Guido Governatori$^1$, 
 Matteo Cristani$^2$,
 Nick van Beest$^1$ and
 Silvano Colombo-Tosatto$^1$\\
 $^1$Data61, CSIRO, Australia\\
 $^2$University of Verona, Italy
}
\institute{}

\begin{document}

\maketitle

\begin{abstract}
	
Linear Logic and Defeasible Logic have been adopted to formalise different
features relevant to agents: consumption of resources, and reasoning with
exceptions. We propose a framework to combine sub-structural features,
corresponding to the consumption of resources, with defeasibility aspects,
and we discuss the design choices for the framework.

\end{abstract}

\section{Introduction}\label{section:intro}

Many different models of agency have been proposed over the years. In some of
those understandings, agents are assumed to be rational entities capable to
reason about the environment in which they are situated, and to deliberate
about what actions to take to achieve some particular goals.

Many logic-based approaches have been proposed to account for the rational
behaviour of an agent. For example, in the well known BDI architecture (and
architectures inspired by it), agents first deliberate about the goals to
achieve and, based on such goals, they select the plans to implement from
their plan libraries. Finally, during or after the execution of the plans,
the agents receive feedback from the environment, which can trigger the
so-called \emph{reconsideration}: the activity to determine whether the
intended goals are still achievable with the selected plan and the current
state of execution.

Most of the logic-based approaches take an idealised representation: the
agents have unlimited reasoning power, complete knowledge of the environment
and their capabilities, and unlimited resources. Over the years, a few
approaches (using different logics) have been advanced to overcome some of
these ideal (unrealistic) assumptions. 

In \cite{linear-agent,linear-agent-Mel1,linear-agent-Mel2}, the authors
propose the use of Linear Logic to model the notion of resource utilisation,
and to generate which plans the agent adopts to achieve its goals. In the same
spirit, \cite{jaamas:bio,tplp:goal} address the problem of agents being able
to take decisions from partial, incomplete, and possibly inconsistent
knowledge bases, using (extensions of) Defeasible Logic (a computational and
proof theoretic approach) to non-monotonic reasoning and reasoning with
exceptions. 
While these last two approaches seem very far apart, they are both based on
proof theory (where the key notion is on the idea of (logical) derivation),
and both logics (for different reasons and different techniques) have been
used for modelling business processes
\cite{Kanovich199748,Tanabe1997156,DBLP:journals/apal/EngbergW97,rao2006composition,DBLP:conf/prima/OlivieriGSC13,DBLP:conf/prima/OlivieriCG15,DBLP:conf/edoc/GhooshchiBGOS17}.

Formally, a business process can be understood as a compact representation of
a set of traces, where a trace is a sequence of tasks. A business process is
hence equivalent to a set of plans with possible choices. The idea behind
the work mentioned above is to allow agents to use their deliberation phase to
determine the business processes (instead of the plans) to execute.

In this paper, we discuss some design choices (and offer some results) about
the combination of linear logic (or more in general, \emph{sub-structural
logic}) and a computationally oriented non-monotonic formalism. To the best
of our knowledge, this is the first investigation in this area, but it
combines two approaches that have proved useful for modelling some aspects of
agency.

We expect this work to be foundational for further research in modelling
agents and the way agents create their plans during the deliberation phase,
taking into account the utilisation of resources and possible exceptions (or
partial knowledge of the environment). Hereafter, we focus on introducing the
key logical aspect to be examined in the paper.

Logic is often described as the ``art'' of reasoning, or in other terms, its subject matter is how
to derive conclusions from given premises. Under this prospective we can distinguish rules (or
sequents, or instances of a consequence relation) and inference (or derivation) rules.  A rule
specifies that some consequences follow from some premises, while a derivation rule provides a
recipe to determine the valid steps in a proof or derivation.  A classical example of a derivation
rule is Modus Ponens (i.e., from `$\alpha\rightarrow\beta$' and $\alpha$  to derive $\beta$). A rule
can be understood as a pair 
\[
  \Gamma\vdash \Theta,
\] 
where $\Gamma$ and $\Theta$ are collections of
formulas in an underling language. In Classical Logic, $\Gamma$ and $\Theta$ are sets of formulas
and in Intuitionistic Logic $\Theta$ is a singleton. Thus the rules 
\[
  \alpha,\beta \vdash \gamma,\delta \qquad \text{ and } \qquad
  \beta,\alpha \vdash \delta, \gamma
\] 
are the same rule. Where ``,'' in the antecedent
$\Gamma$ is understood as conjunction and disjunction in the consequent $\Theta$. In substructural
logics (e.g., Lambek Calculus, and the family of Linear Logics), $\Gamma$ and $\Theta$ are assumed
to have an internal structure, and they are considered as multi-sets or sequences. An interpretation
of a rule is how to transform the premises in the conclusion. Thus the rule 
\[
  \alpha,\beta,\alpha\vdash \gamma
\] 
can be taken to mean that we need two instances of $\alpha$ with an instance of
$\beta$ in between to produce an instance of $\gamma$.
Derivation rules, on the other hand, tell us how to combine rules, to obtain new rules. For example,
the derivation rule \[
\frac{\Gamma\vdash\alpha\qquad \Theta\vdash\alpha\rightarrow\beta}
     {\Gamma,\Theta\vdash\alpha,\alpha\rightarrow\beta,\beta}
\] establishes whether we have a derivation of $\alpha$ from $\Gamma$ and a derivation of
$\alpha\rightarrow\beta$ from $\Theta$, then we can combine the $\Gamma$ and $\Theta$, to obtain a
derivation, where we have $\alpha$ followed by $\alpha\rightarrow\beta$ and then $\beta$. If the
formulas denote activities (or tasks) and resources, then the consequent is a sequence of tasks
describing the activities to be done (and the order in which they have to be executed) to produce an
outcome (and also, what resources are needed). Thus, we can use the rules to model transformation in
a business processes, and derivations as the traces of the process (or the ways in which the process
can be executed or the runs of system).

A formalism that properly models processes should feature some key
characteristics, and one of the most important ones is to identify which
resources are \emph{consumed} after a task has finished its execution.
Consider the notorious vending machine scenario, where the dollar resource is
spent to \emph{produce} the can of cola. Trivially, once we get the cola, the
dollar resource is no longer spendable (unless it can be, somehow,
\emph{replenished}). However, the specifications of a process may include
thousands of rules to represent at their best all the various situations that
may occur during the execution of the process itself: situations where the
information at hand may be incomplete and, sometimes, even contradictory, and
rules encoding possible exceptions. This means that we have to adopt a
formalism that is able to represent and reason with exceptions, and partial
information.

Defeasible Logic (DL) \cite{nute2} is a non-monotonic rule based formalism, that
has been used to model exceptions and processes. The starting point being that,
while rules define a relation between premises and conclusion, DL takes the stance that multiple relations are possible, and it focuses
on the ``strength'' of the relationships. Three relationships are identified:
\emph{strict rules} specifying that every time the antecedent holds the
consequent hold; \emph{defeasible rules}, when the antecedent holds then,
normally, the consequent holds; and \emph{defeaters} when the antecedent
holds the opposite of the consequent might not hold. An example of rules with
a baseline condition and exception is the scenario the outcome of inserting a
dollar coin in a vending machine is that we get a can of cola, unless the
machine is out of order, or the machine is switched off. Thus, we can
represent this scenario with the rules\footnote{$r_i$ is the name of rule $i$, symbol $\To$ denotes rules meant to derive \emph{defeasible} conclusions, i.e. conclusions which may be defeated by contrary evidence. As it will be clear in Section~\ref{sec:LDL}, defeasible rules, defeaters, and the superiority relation represent the non-monotonic aspects of our framework.}:
\begin{gather*}
r_1: 1\$ \Rightarrow\mathit{cola}\\
r_2: \mathit{OutOfOrder} \Rightarrow \neg\mathit{cola}\\
r_3: \mathit{Off} \Rightarrow \neg\mathit{cola}.
\end{gather*}
Based on the discussion so far, the motivation of the paper is twofold. From
a technical point of view we want to combine, from a logic perspective, the
mechanisms of defeasibility with mechanisms from substructural logic (to
capture the order of resources, and the consumption of resources). It is
clear that the resulting combination of logical machinery could provide a
much better formalism for the representation of processes. Accordingly, we
will use the point of view of business process modelling to illustrate the
technical features we are going to define in the logic (or, to be more
specific, for possible variants of substructural defeasible logics).

The remainder of this paper is structured as follows.
Section~\ref{sec:desired_properties} introduces the reader to the features we
want our logics to be equipped with. Subsequently, we provide the
formalisation of the logics in Section~\ref{sec:LDL}, and finally,
Section~\ref{sec:Conclusions} concludes our work.

\section{Desired properties} 
\label{sec:desired_properties}

We dedicate this section to detailing which new features our logics need to
implement and, for each of them, to justify their importance with respect to
real life problems.

\paragraph{Ordered list of antecedents}\mbox{}\\
Given the rule `$r: A, B \To C$', the order in which we derive $A$ and $B$ is typically irrelevant for the derivation
of $C$. As such, $r$ may indistinctively assume the form `$B, A \To C$'.
Consider a login procedure which requires a username and password.
Whether we insert one credential before the other does not affect a
successful login.

Nonetheless, sometimes it is meaningful to consider an \emph{ordered}
sequence of atoms in the head of a rule, instead of an \emph{unordered} set
of antecedents. 
Suppose we have the two activities `\emph{Check Creditworthiness}' and
`\emph{Approve Loan}'. Neither of them depends on the other. However,
performing one activity before the other may affect the final result: if we
approve the loan before creditworthiness has been checked and approved, then
a loan may potentially be provided to someone who is not able to repay.

This allows us to capture the fact that some resources may be
\emph{independent} of each other from the derivational viewpoint (one does
not \emph{derive} the other), but are \emph{dependent} from a temporal perspective
(one must be \emph{obtained} before the other).
Naturally, in the same set/list of antecedents, combinations of unordered and
ordered sequences of literals is possible. For
instance, \[r: A; B; (C, D); E\To F\] represents a situation where, in order
to obtain $F$, we need to first obtain $A$, then $B$, then either $C$ or $D$
in any order, and lastly, only after both $C$ and $D$ are obtained, we need
to obtain $E$. The notation `;' is used as a separator between
elements in an ordered sequence, while `,' separates unordered sequences.

\paragraph{Multi-occurrence/repetitions of literals}\mbox{}\\
From these ideas, it follows that some literals may appear
in multiple instances, and that two rules such as 
\[
r: A; A; B \To C \qquad \text{ and } \qquad 
s: A; B; A \To C
\] 
are semantically different. For instance, rule $r$ may describe a scenario where the order of a product may require two deposit payments followed by a full payment prior to delivery.
%
%
%
%
Regarding $s$, consider that $A$ is now `Add a tablespoon of ice sugar' and $B$ is `Stir for 1 minute`. A perfect frosting requires many repetitions of $A$ after $B$ after $A$, for a specific number of repetitions.

\paragraph{Resources consumption}\mbox{}\\
Assume we have two rules, 
\[
r: A, B \To D \qquad \text{ and }\qquad
s: A, C \To E.
\] 
If we are able to derive $A$, $B$ and $C$, then $D$
and $E$ are subsequently obtained. Deducing both $D$ and $E$ is a typical problem of
\emph{resource consumption}. 

Given the financial state of a customer (i.e., their pay cheque and their
monthly spending), a finance approval is sent to the customer for the
requested loan. However, that finance approval can only be used once, given
the financial situation of that customer. That is, they cannot obtain another
loan with the same finance approval. If the customer wants to apply for
another loan, they are required to obtain a new finance approval first.

This example indicates that some literals represent
resources that are \emph{consumed} during the derivation process: if they appear in the antecedent of a rule, and such a rule produces
its conclusion, then the other rules with the same literals in their
antecedent can no longer fire (unless there are multiple occurrences).

Conversely, some resources are \emph{not} consumed once used. For
instance, a policy at a bank may dictate that a customer has to be below
65 years old to be eligible for a mortgage. A similar requirement
may hold for a car loan. However, a customer may apply for both a mortgage
and a car loan, as neither of these applications invalidate the fact that the
customer is younger than 65 years old. That is, the information regarding the
customers' age is not consumed when used.


The discussion of when a resource has to be
considered consumable/non-consumable is outside the scope of this paper. It is a duty of the knowledge engineer to decide whether to tag a resource as consumable, or non-consumable. For the remainder of this paper, we assume all literals
to be consumable. The treatment/derivation of non-consumable literals is the same as in Standard Defeasible Logic (SDL), and thus something well known in the literature of SDL. 

\paragraph{Concurrent production}\mbox{}\\
Symmetrically, we consider two distinct rules having the same conclusion:
\[
r: A \To C \qquad \text{ and }\qquad
s: B \To C.
\] 
It now seems reasonable that, if both $A$
and $B$ are derived, then we conclude two instances of $C$ (whereas in
classical logics we only know that $C$ holds true). For example, consider a
family where it is tradition to have pizza on Friday evening. Last Friday,
the parents were unable to communicate with each other during the day, and one
baked the pizza while the other bought take-away on the way home.


However, there exists consistent cases where multiple rules for the same literal produce only
\emph{one} instance of the literal (even if they all fire). For example, both a digital or handwritten signature would provide permission to proceed with a request. The same
request does not require permission twice: either it is permitted, or it is not.

\paragraph{Resource consumption: A team defeater perspective}\mbox{}\\
Sceptical logics provide a means to decide which conclusion to draw in case of
contradicting information. Typically, a superiority relation is given among
rules for contrary conclusions: it is possible to derive a conclusion only if
there exists a \emph{single} rule stronger than \emph{all} the
rules for the opposite literal.

Defeasible Logic handles conflicts differently, and the idea here is that of
\emph{team defeater}. We do not look at whether there is a single rule
prevailing over all the other rules, but rather whether there exists a
\emph{team} of rules which can jointly defeat the rules for the contrary
conclusion. That is, suppose rules $r'$, $r''$ and $r'''$ all conclude $P$,
whilst $s'$ and $s''$ are for $\neg P$. If $r'>s'$ and $r''>s''$, then the
team defeater made of $\set{r',r''}$ is sufficient to prove $P$.

The focus remains on resource consumption and production. As such, the
questions we need to answer are, again, which resources are consumed, and
how many instances of the conclusion are derived.
We start by distinguishing the two scenarios where: (a) neither of the teams
prevail, (b) one team wins.
Consider 
\[
r': A \To P,\quad 
r'': B \To P,\quad 
r''': C \To P,\quad 
s': D \To \neg P,\quad
s'': E \To \neg P.
\] 
In case (a), e.g., when no superiority
is given, we cannot conclude for either conclusion. Hence, the question is
``Will any of the resources be consumed?''.
In case (b), we assume $r'>s'$ and $r''>s''$, and we conclude that $P$. How
many instances of $P$ are produced? One solution is to produce three
instances of $P$ and, accordingly, $A$, $B$ and $C$ are all consumed. We can
instead consistently assume that we produce $P$ twice, through the two
winning rules $r''$ and $r'''$ only, but not via $r'$; we thus consume $B$
and $C$, but \emph{not} $A$. 

Lastly, on the perspective of the \emph{defeated} rules another relevant
question is: Are $D$ and $E$ ever consumed? As clear, there is no unique
answer. There are consistent scenarios where the literals in the
\emph{defeated} rules are consumed, and other cases where they are
not. In Section~\ref{sec:LDL}, we provide different solutions to cover the
various cases.

Consider the process of writing a scientific publication for a conference. If
the paper is accepted, the \emph{manuscript} resource is consumed, since it
cannot be submitted again. On the contrary, if the paper is rejected, the
\emph{manuscript} resource is \emph{not} consumed since it can be submitted
again to other venues.

\paragraph{Multiple conclusions and resource preservation}\mbox{}\\
Consider internet shopping. As soon as we pay for our online order, the bank
account balance decreases, the seller's account increases. Both the seller
and the web site have the shipping address and, possibly, the credit card
number.

The conclusion of a rule is usually a single literal. The above example
suggests that a single rule may produce more than one conclusion, which
cannot be represented by multiple rules with the same set of antecedents. For example, consider the rules 
\[
r: A, B \To C \qquad\text{ and }\qquad 
s: A, B \To D.
\] 
In a
propositional calculus, once the system derives $A$ and $B$, by Modus Ponens,
we obtain both $C$ and $D$. However, when we consider resource consumption,
then it is clear that only one rule can produce its conclusion, whilst the
other cannot.
We tackle this problem by allowing rules to have multiple conclusions. Thus,
$r$ and $s$ can be merged into the single rule 
\[r': A, B \To C, D.\]

Similar to our discussion 
on the ordering of antecedents, we may have any
combination of ordered/unordered sequences of conclusions. In the previous
example, only after we have provided the credit card credentials, our bank
account decreases, whilst we can provide the shipping address before the
credit card credentials, or the other way around.

The notion of multiple conclusions, along with the discussion 
on team defeaters, leads to another problem. Consider the two
rules
\[ 
r: A \To B;C;D \qquad \text{ and }\qquad
s: E \To \neg C,
\] 
where no superiority is given. Do we
conclude that $B$ or $D$? Moreover, what happens if now we have `$r: A \To B,C,D$' and we establish that
$s$ is stronger then $r$? Do we conclude that $B$ and $D$
(meaning that only the derivation of $C$ has been blocked by $s$), or will
the production of $B$ and $D$ be affected also?

\paragraph{Loops}\mbox{}\\
The importance of being able to properly handle loops is evident: loops play a
fundamental role in many real life applications, from business processes to
manufacturing. Back to the login procedure, if one of the credentials is
wrong, the process loops back to a previous state, for instance, by asking
the user to re-enter both credentials.

Naturally, a system is able to properly handle loops when it can
handle/recognise the so-called \emph{exit conditions}, to prevent infinite
repetition of the same set of events. For example, after three wrong login
attempts, the login procedure may prevent the user from further attempts and
require them to undergo a \emph{retrieve credential} procedure.


\section{Language and logical formalisation of RSDL}\label{sec:LDL}

Before introducing the notation used throughout the remainder of this paper, we will first justify a
number of implementation choices.
In Section~
\ref{sec:desired_properties}, we stated that there are two \emph{types} of
literals: consumable against non-consumable. 
In addition, we
introduced the notion of ordered sequences of
literals in the antecedents and conclusions, we stated that any
combination of ordered or unordered sequences of literals is (theoretically)
possible.

The logics presented here shall \emph{only} consider: (i) \emph{consumable}
literals, and (ii) either multi-sets or ordered sequences of literals (but
not their combination). Our motivation is that adding conditions to the proof
tags (labels that describe how a literal can be derived) in order to deal
with non-consumable literals and alternating ordered/unordered sequences is a
trivial and pedant task. Their formalisation is exactly the same of that in
Standard DL (SDL), and thus the process would not add any value (it will be
done for the sake of completeness in future work).

Our logics deal with two types of derivations: \emph{strict} and
\emph{defeasible}. Their underlying meaning is the same as those in SDL.
\emph{Strict rules} derive indisputable conclusions, i.e., conclusions that
which are always true. Thus, if two strict rules have opposite conclusions,
then the resulting logic is inconsistent. On the contrary, defeasible rules
are to derive pieces of information that can be defeated by contrary
evidence, like `Birds typically fly' since we know that `Penguins are birds
that do not fly'. Finally, defeaters are special type of rules whose only
purpose is to block contrary evidence. They cannot be used to directly derive
conclusions, but only to prevent other rules to conclude.

We have now passed through the conceptual basis of our logics and can move
forward to introduce the language of Resource-driven Substructural Defeasible Logic (RSDL). 
We will use greek letters to denote propositional atoms, while roman letters
$r$, $s$, $t$ are reserved to denote rule labels. In addition, $l$ and $c$
are reserved to denote \uline{l}ines and \uline{c}olumns in a proof, while
other roman letters are typically used as subscripts to denote the
cardinality of sets/sequences; capital $F$ denotes the set of facts,
capital $R$ the set of rules, $A(r)$ the list of antecedents, and $C(r)$
the set/list of conclusions of rule $r$.

$\PROP$ is the set of propositional atoms, the set $\LIT= \PROP \cup
\set{\neg \varphi | \varphi \in \PROP}$ denotes the set of literals. The
\emph{complement} of a literal $\varphi$ is denoted by $\non \psi$; if $\psi$
is a positive literal $\varphi$, then $\neg \psi$ is $\neg \varphi$, and if
$\psi$ is a negative literal $\neg\varphi$, then $\non\psi$ is $\varphi$.

We adopt the standard DL definitions of \emph{strict rules}, \emph{defeasible
rules}, and \emph{defeaters}.

\begin{definition}

Let $\LAB$ be a set of arbitrary labels. Every rule is of the type $r: A(r) \hookrightarrow C(r)$, where

	\begin{enumerate}
		\item $r\in \LAB$ is a unique name.
		\item\label{Alist} $A(r)=\alpha_1,\dots,\alpha_n$, the \emph{antecedent}, or \emph{body}, of the rule is a list of literals.
		
		
		\item An \emph{arrow} $\hookrightarrow\in \set{\to, \To,
	\leadsto}$ denoting a strict rule, a defeasible rule, and a
	defeater, respectively;

		\item $C(r)$ is the \emph{consequent}, or \emph{head}, of the rule. For the head, we consider three options:
	   \begin{enumerate}
	     \item The head is a single literal $\varphi$;
	     \item The head is a list $\varphi_1,\dots,\varphi_m$ (to be understood as a multi-set);
	     \item The head is a list $\varphi_1;\dots;\varphi_m$ (to be understood as a ordered list).
	   \end{enumerate}
	\end{enumerate}
\end{definition}

With abuse of notation, we will often refer to
$\varphi_1,\dots,\varphi_m$ as a set, and overload standard set theoretic
notation.
Given a set of rules $R$, and a rule $r: A(r) \hookrightarrow C(r)$ we use
the following abbreviations for specific subsets of rules: (i) $R_s$ is the
subset of strict rules, (ii) $R_{sd}$ is the set of strict and defeasible
rules, (iii) $R[\varphi;i]$ is the set of rules where $\varphi$ appears at
index $i$ in the consequent where the consequent is a list, (iv)
$R[\varphi,i]$ when the consequent is a set containing $\varphi$.

\begin{definition}
	A \emph{resource-driven substructural defeasible (rsd) theory} is 
	a tuple $(F, R, \succ)$ where:
(i) $F\subseteq \LIT$ are pieces of information denoting the (consumable)
resources available at the beginning of the computation. This differs
strikingly from SDL, where i) they denote \emph{always-true} statements; (ii) $R$
is the set of rules; (iii) $\succ$, the superiority relation, is a binary
relation over $R$.
\end{definition}

A theory is \emph{finite} if the set of facts and rules are finite.
In SDL, a \emph{proof} $P$ of length $n$ is a finite sequence
$P(1),\dots,P(n)$ of \emph{tagged literals} of the type $\pm\Delta\varphi$
and $\pm\partial\varphi$. The idea is that, at every step of the
derivation, a literal is either proven or disproven.

In our logic, we must be able to derive multiple conclusions in a single
derivation step, and hence we require a mechanism to determine when
premises have been used to derive conclusion. Accordingly, we modify
the definition of proof to be a matrix.

\begin{definition}\label{def:Proof}
	A proof $P$ in RSDL is $P(l,c)$ a finite matrix
$P(1,1),\dots,P(l,c)$ of \emph{tagged literals} of the type $\pm\Delta\varphi$, $\pm\partial\varphi$, 
$+\sigma\varphi$, $+\Delta\varphi^\checkmark$ and $+\partial\varphi^\checkmark$.
	
\end{definition}

We assume that facts are
simultaneously true at the beginning of the computation. Notation $+\#\varphi^\checkmark$, $\#\in\set{\Delta,\partial}$, denotes the
fact that $\varphi$ has been consumed. The distinctive notation for when a
literal is proven and when it is consumed will play a key role to determine
which rules are applicable.

The tagged literal $\pm\Delta\varphi$ means that $\varphi$ is \emph{strictly
proved/refuted} in $D$, and, symmetrically, $\pm\partial\varphi$ means that
$\varphi$ is \emph{defeasibly proven/refuted}. The set of positive and negative conclusions is called \emph{extension}.

In SDL, given a set of facts, a set of rules and a superiority
relation, the extension is unique. It is clear that this is not the case when
resource consumption and ordered sequences are to be taken into account:
depending on the order in which the rules are applied, (rather) different extensions
can be obtained. In sub-structural defeasible logic every distinct derivation
corresponds to an extension.

In SDL, derivations are based on the notions of a rule
being \emph{applicable} or \emph{discarded}. Briefly, a rule is applicable
when every literal in the antecedent has been proven at a previous step. We report hereafter a standard defeasible proof tag.

\begin{tabbing}
If \=$P(n+1) = +\partial\varphi$ then\+\\
  (1) $\exists r\in R_{sd}[\varphi]$: $r$ is applicable and\\
  (2) \=$\forall s\in R[\non\varphi]$ either\+\\
      (2.1) $s$ is discarded or\\
      (2.2) \=$\exists t\in R[\varphi]$: $t$ is applicable and $t\succ s$.
\end{tabbing}

A literal is defeasibly proven when there exists an applicable rule for such
a conclusion and all the rules of the opposite are either discarded, or
defeated by stronger rules. (Strict derivations only differ in that, when a
rule is applicable, we do not care about contrary evidence, and the rule will
always produce its conclusion nonetheless. As such, the consistency of the
logics depends only on the strict part of the logics.)

As for SDL, we obtain variants of the logic by providing different
definitions of being \emph{applicable} and \emph{discarded}. More
specifically, for RSDL, definitions of applicability and discardability need
to take into account (i) the number of times a literals appear in the body of
a rule, (ii) how many times they have been derived\footnote{Trivially, e.g.,
if literal $\alpha$ has been derived twice, but it appears in the antecedent
of three rules, only two of such rules can produce their conclusions.}, (iii)
the order in which the literals occur in the body of a rule and in a
derivation. In addition, we have to extend the structure of the proof
conditions to include mechanisms or conditions to determine when
literals/resources have been used to derive new literals/resources.

We shall proceed incrementally. First, we provide definitions for multi-sets.
We then provide definitions for sequences. In both cases, we consider
rules with a single literal for conclusion. Consequently, we end with the definitions
to describe multiple conclusions.

\begin{definition}\label{def:ApplicableMS}
  A rule $r$ is $\#$\emph{-applicable}, $\#\in\set{\Delta,\partial,\sigma}$, at $P(l+1,c+1)$ iff 
	$\forall$ $\alpha_i\in A(r)$ then $+\#\alpha_i\in P[(1,1)..(l,c)]$.
Moreover, we say that $r$ is $\#$\emph{-consumable} iff $r$ is $\#$-applicable and $\exists$ $l'\leq l$ such that $P(l',c)=+\# \alpha_i$.
\end{definition}

A rule is consumable if it is applicable and, for every literal in its
antecedent, there is an \emph{unused} occurrence. (This can be done by
checking the previous derivation step $c$.)
Discardability is obtained by applying the principle of the strong negation to the definition of applicability.

\begin{definition}\label{def:DiscardedMS}
  A rule $r$ is $\#$\emph{-discarded}, $\#\in\set{\Delta,\partial}$, at $P(l+1,c+1)$ iff
	$\exists$ $\alpha_i\in A(r)$ such that $-\#\alpha_i\in P[(1,1)..(l,c)].$
Moreover, we say that $r$ is $\#$\emph{-non--consumable} iff either $r$ is discarded, or $\forall$ $l'\leq l$, $P(l',c)\neq +\partial \alpha_i$.
\end{definition}

	


	

Lastly, we define the conditions describing when a literal is consumed.

\begin{definition}\label{def:ConsumedLiteral}
  Given rule $r$, a literal $\alpha\in A(r)$ is $\#$-consumed, $\#\in\set{\Delta,\partial}$, at $P(l+1,c+1)$, iff
\begin{enumerate}
	\item $\exists$ $l'\leq l$ such that $P(l',c)=+\#\alpha$, and
	
	\item $P(l',c+1)=+\#\alpha^\checkmark$.
\end{enumerate}
\end{definition}

The following example 
illustrates how we use $+\partial \varphi^\checkmark$ within the resource consumption mechanism.

\begin{example}
	Consider $D=(\set{\alpha},R, \emptyset)$, where 
\[
	R=\{r_0: \alpha \To \beta,\quad r_1: \beta \To \varphi,\quad r_2: \beta \to \psi\}
\]	
and (one of) the corresponding proof table(s):
\[
\begin{array}{r|ccc}\label{tab:1}
 P & 1 & 2 & 3\\\hline		
1 & +\Delta\alpha  & +\Delta\alpha^\checkmark  & +\Delta\alpha^\checkmark \\	

2 &         & +\partial\beta              & +\partial\beta^\checkmark\\

3 &         &                    & +\partial\varphi\\
\end{array}
\]
\end{example}

Naturally, two mutually exclusive extensions are possible, based on whether
$+\partial\beta$ is used by $r_1$ to derive $+\partial\varphi$, or by $r_2$ to derive
$+\partial\psi$. Table~\ref{tab:1} shows the former case. At $P(1,1)$ we
obtain $+\Delta\alpha$, instance that is consumed in deriving $+\partial\beta$ at $P(2,2)$. Thus, $P(1,2)=+\Delta\alpha^\checkmark$. Symmetric situation for activating $r_1$ at the third
derivation step, which results in $P(2,3)=+\partial\beta^\checkmark$ and
$P(3,3)=+\partial\varphi$. Now, at the fourth derivation step, $r_2$ is
applicable but non-consumable, and hence we cannot derive $+\partial\psi$.

\subsection*{Proof tags for multi-sets in the antecedent and single conclusion} 
\label{sec:proof_tags_for_multi_sets_in_the_antecedent_and_single_conlusion}


We now present the proof tags, and we begin with: (1) the antecedent is a multi-set, (2) single literal in the conclusion. $+\Delta$ describes positive definitive (strict) derivations.

\begin{tabbing}
 $+\Delta$: \= If $P(l+1,c+1)=+\Delta \varphi$ then\+\\
 (1) \=$\varphi\in F$, or\\
 (2) \= (1) $\exists r\in R_s[\varphi]$ such that\\
 \> (2) $r$ is $\Delta$-consumable and\\
 \> (3) $\forall \alpha_j\in A(r)$, $\alpha_j$ is $\Delta$-consumed.
\end{tabbing}
 
Literal $q$ is definitely provable if it either is a fact, or there is a strict,
applicable rule for $\varphi$, whose antecedent literals can be consumed.
Condition (2) actually consumes the literals by replacing $+\Delta\alpha_j$
with $+\Delta\alpha_j^\checkmark$. Proof tag $-\Delta$ is as follows:

\begin{tabbing}
 $-\Delta$: \= If $P(l+1,c+1)=-\Delta \varphi$ then\+\\
 (1) \=$\varphi\notin \FACTS$ and\\
 (2) \= $\forall r\in R_s[\varphi]$, $r$ is $\Delta$-non-consumable.

\end{tabbing}

Literal $q$ is definitely refuted if $\varphi$ is not a fact, and every
rule for $\varphi$ is non-consumable.
We can now turn our attention to the definition of the proof tags for
defeasible conclusions. In particular, we provide proof conditions
(corresponding to inference rules) for three types of conclusions:
$+\partial\varphi$ meaning that the current derivation $P$ proves $\varphi$;
$-\partial\varphi$ meaning that the current derivation $P$ prevents the proof
of $\varphi$, or in other terms that $\varphi$ is refuted; and
$+\sigma\varphi$ whose intuitive reading is that the $\varphi$ would be
derivable in the current proof if more (appropriate) resources are would be
available. Alternatively, for $+\sigma\varphi$ indicates that there are
applicable rules for $\varphi$, but the resources for such rules have been
used to derive other conclusions.

\begin{tabbing}
 $+\partial$: \= If $P(l+1,c+1)=+\partial \varphi$, then\+\\
 (1) \=$+\Delta \varphi\in P(l,c)$ or\\
 (2) \= (1) $-\Delta\non \varphi \in P(l,c)$ and\+\\
     (2) \= $\exists r\in R_{sd}[\varphi]$ $\partial$-consumable and\\
     (3) \= $\forall s \in R[\non \varphi]$ either\+\\
 	     (1) \= $s$ is $\partial$-discarded, or\\
	     (2) $\exists t\in R[\varphi]$ $\partial$-consumable, $t\succ s$, and\\
 	     (3) if $\exists w\in R[\non\varphi]$ $\partial$-applicable, $t\succ w$, then\+\\
             (1) \=$\forall \alpha_j\in A(t)$, $\alpha_j$ is $\partial$-consumed, otherwise\\
            (2) $\forall \alpha_k\in A(r)$, $\alpha_k$ is $\partial$-consumed.
\end{tabbing}

Condition (1) allows us to inherit a defeasible derivation from a definite
derivation. Condition (2.1) ensures that the logic is sound (i.e., that it is
not possible to derive $\varphi$ and its negation as provable defeasible
conclusions unless they are derivable from the strict part only. Condition
(2.2) requires that there is rules that is triggered by previously proved
literals that have not been used to trigger other conclusions. Clause (2.3.1)
is the standard one of defeasible logic, meaning that to rebut an attacking
argument, we can show that some of the premises of the argument/rule have
been refuted. The second method to rebut the attacking arguments (rules for
$\non\varphi$) is to show that they are defeated, i.e., weaker than
applicable rules for the conclusion we want to prove, similarly the
antecedents of such rules must not have been used for other conclusions,
clause (2.3.2). The final part is the mechanism to determine which resources
are consumed by the derivation of $\varphi$. If there are no applicable rules
for $\non\varphi$ the resources are taken for the rule proposed as an
argument, that is rule $r$ (2.3.3.2); if there are counter-argument, the
resources are taken from each rule rebutting a counter-argument, for each
possible applicable counterargument (2.3.3.1). Note, that in this way we
capture the idea of team defeat (see Section~
\ref{sec:desired_properties}.

For $-\partial$ we use the strategy similar to that used in \cite{ecai2000-5}
to provide proof condition for the ambiguity propagating variant of SDL, that is, we make it easier to attack a rule (2.2.2). Also, for
the derivation of literals tagged with $-\partial$ does not require the
consumption of resources. Resources are consumed only when we positively
derive literals.

\begin{tabbing}
$-\partial$: \= If $P(l+1,c+1)=-\partial\varphi$, then \+\\
  (1) $-\Delta\varphi\in P[l,c]$ and\\
  (2)\=(1) $+\Delta\non\varphi\in P[l,c]$ or\+\\
     (2) \=$\forall r\in R_{sd}[\varphi]$ either\+\\
        (1) $r$ is $\partial$-discarded or\\
        (2) \=$\exists s\in R[\non\varphi]$ such that\+\\
            (1) $s$ is $\sigma$-applicable and\\
            (2) \=$\forall t\in R[\varphi]$ either\+\\
                (1) $t$ is $\partial$-discarded or
                (2) not $t\succ s$.  
\end{tabbing}

The idea behind $+\sigma$ is that there are rules applicable for the
consequent, irrespective whether the premises have been used or not. However,
we have to check that the rule is not defeated by an applicable rule for the
opposite (2.2).

\begin{tabbing}
$+\sigma$: \=If $P(l+1,c+1)=+\sigma\varphi$ then\+\\
 (1) $\Delta\varphi\in P[l,c]$ or\\
 (2) \=$\exists r\in R_{sd}[\varphi]$ such that \+\\
     (1) $r$ is $\sigma$-applicable and\\
     (2) \=$\forall s\in R[\non\varphi]$ either\+\\
         (1) $s$ is $\partial$-discarded or
         (2) not $s\succ r$.
\end{tabbing}

\begin{example}\label{ex:MultisetSingleConclusion}
	Consider $D=(\set{\alpha,\beta,\gamma},R,$ $\succ=\set{(r_2,r_3)})$, where 
\[	
R=\{r_{0}: \mathit{\alpha} \To \mathit{\varphi},\quad r_{1}: \mathit{\alpha} \To \mathit{\psi},\quad r_{2}: \mathit{\beta} \To \mathit{\varphi},\quad r_{3}: \mathit{\gamma} \To \non\mathit{\varphi}\},
\]
and the corresponding proof table:
\[
\begin{array}{r|ccccc}\label{tab:2}
 P & 1 & 2 & 3 & 4 & 5\\\hline		
1 & +\Delta\alpha & +\Delta\alpha & +\Delta\alpha  & +\Delta\alpha  & +\Delta\alpha^\checkmark \\	

2 &  & +\Delta\beta & +\Delta\beta & +\Delta\beta^\checkmark & +\Delta\beta^\checkmark\\

3 &  &  &+\Delta\gamma & +\Delta\gamma & +\Delta\gamma\\

4 &  &  &           & +\partial\varphi        & +\partial\varphi\\

5 &  &  &  &                                 & +\partial\psi
\end{array}
\]
\end{example}

Assume that, at the fourth derivation step, $r_0$ is taken into
consideration; $r_0$ is actually consumable, but so is $r_3$, and no
superiority is given between $r_0$ and $r_3$. There actually exists a
consumable rule stronger than $r_3$, $r_2$. Accordingly, the team
defeater allows us to prove $+\partial\varphi$, and only $\beta$ is consumed
in this process. This implies that $\alpha$ is still available,
and can be used at the fifth derivation step to produce $+\partial\psi$, via
$r_1$.

\subsection*{Proof tags for sequences in the antecedent and single conclusion} 
\label{sec:proof_tags_for_sequences_in_the_antecedent_and_single_conlusion}


We can now move forward and consider sequences in the antecedent. Naturally,
definitions of being applicable, discarded, and consumable must be revised.

A rule is \emph{sequence applicable} when the derivation order reflects the
order in which the literals appear in the antecedent, i.e., for every two
literals in the antecedent, say $\alpha$ and $\beta$, such that $\alpha$
appears before $\beta$, there exists a derivation of $\alpha$ before every
derivation (for that occurrence\footnote{In order to handle situations like $A(r)=\set{\alpha;\beta;\alpha;\beta}$
(as illustrated in Example~\ref{ex:ABA}).}) of $\beta$.

\begin{definition}\label{def:SeqAppl}
	A rule $r\in R[\varphi]$ is $\#$-\emph{sequence applicable}, $\#\in\set{\Delta,\partial}$, at $P(l+1,c+1)$ iff for all $\alpha_i\in A(r)$:
	
\begin{enumerate}

	\item\label{cond:Appl1} there exists $c_i \leq c$ such that $P(l_i,c_i) = +\# \alpha_i$, $l_i \leq l$, and 
	
	\item for all $\alpha_j\in A(r)$ such that $i < j$, then
	
	\item\label{cond:Appl3} for all $c_j \leq c$ such that $P(l_j,c_j) = +\# \alpha_j$, $l_j \leq l$ then $l_i < l_j$ and $c_i < c_j$.

\end{enumerate}

We say that $r$ is $\#$-\emph{sequence consumable} iff is $\#$-\emph{applicable} and 4. $P(l_i,c)=+\# \alpha_i$.
\end{definition}

A rule is sequence discarded when there exists a literal in the antecedent,
 which has been previously disproven, or there are two proven literals in
the antecedent, say $\alpha$ and $\beta$, such that $\alpha$ appears before
$\beta$, and one proof for $\beta$ is before every proof for $\alpha$.

\begin{definition}\label{def:SeqDisc}
	A rule $r\in R[\varphi]$ is $\#$-\emph{sequence discarded}, $\#\in\set{\Delta,\partial}$, at $P(l+1,c+1)$ iff for there exists $\alpha_i\in A(r)$ such that either
	
\begin{enumerate}

	\item $-\#\alpha_i \in P(l-1,c-1)$, or 
	
	\item for all $c_i \leq c$ such that $P(l_j,c_i)=+\# \alpha_i$, $l_i \leq l$, then
	
	\begin{itemize}
		\item there exists $\alpha_j\in A(r)$, with $i < j$, such that
		
		\item there exists $c_j \leq c$ such that $P(l_j, c_j)=+\# \alpha_j$, $l_j \leq l$, and $c_i > c_j$, $l_i > l_j$.
	\end{itemize}

\end{enumerate}
\end{definition}

The definition of a literal being $\#$-consumed remains the same as before. The proof tags for strict and defeasible conclusions with (i) sequences in the antecedent, and (ii) a single conclusion can be obtained by simply replacing 

\begin{itemize}
	\item $\#$-applicable with $\#$-sequence applicable;
	\item $\#$-consumable with $\#$-sequence consumable;
	\item $\#$-discarded with $\#$-sequence discarded.	
\end{itemize}

\begin{example}\label{ex:ABA}
	Consider $D=(\set{\gamma,\epsilon,
\zeta},R, \succ=\emptyset)$, where $R=\{r_{0}: \mathit{\alpha},\mathit{\beta},\mathit{\alpha} \To \mathit{\varphi},\quad r_{1}: \mathit{\gamma} \To \mathit{\alpha},\quad
	 r_{2}: \mathit{\epsilon} \To \mathit{\alpha},\quad r_{3}: \mathit{\epsilon} \To \mathit{\beta}\}$.

\end{example}

Assume the rules are activated in this order: first $r_1$, then $r_2$, last
$r_3$. Thus, $P(4,4)=+\partial\alpha$, $P(5,5)=+\partial\alpha$, and
$P(6,6)=+\partial\beta$. The derivation order between $\beta$ and the second
occurrence of $\alpha$ has not been complied with, and $r_0$ is sequence
discarded. Same if the order is `$r_3$, $r_1$, $r_2$', whilst `$r_2$, $r_3$,
$r_1$' is a legit order to let $r_0$ be sequence applicable.

\subsection*{Proof tags for sequences in both the antecedent and conclusion} 
\label{sec:proof_tags_for_sequences_in_both_the_antecedent_and_conclusion}

Even when we consider sequences in the consequent, a literal's strict
provability or refutability depends only upon whether the strict rule (where the
literal occurs) is sequence consumable or not. 
As such, given a strict rule $r\in R_s[\varphi;j]$, still $\varphi$'s strict
provability/refutability depends only upon whether $r$ is strictly sequence consumable or not. However, now we also have to verify that, if $r\in R_s[\psi;j-1]$, we prove $\varphi$ immediately after $\psi$. 
The resulting new formalisation of $+\Delta$ is as follows:

\begin{tabbing}
 $+\Delta$: \= If $P(l+1,c+1)=+\Delta \varphi$ then\+\\
 (1) \=$\varphi\in F$, or\\
 (2) \= (1) $\exists r\in R_s[\varphi;j]$ such that\\
 \> (2) $r$ is $\Delta$-sequence-consumable,\\
 \> (3) $r\in R_s[\psi;j-1]$ and $P(l+i-1,c+1)=+\Delta\psi$,\\
 \> (4) $\forall \alpha_j\in A(r)$, $\alpha_j$ is $\Delta$-consumed.
\end{tabbing}

The positive, defeasible proof tag is as follows.

\begin{tabbing}
 $+\partial$: \= If $P(l+i,c+1)=+\partial \varphi$, then\+\\
 (1) \=$+\Delta \varphi\in P(l,c)$ or\\
 (2) \= (1) $-\Delta\non \varphi \in P(l,c)$ and\+\\
     (2) \= (1) $\exists r\in R_{sd}[\varphi;j]$ $\partial$-sequence-consumable and\\
	 	\> (2) $\exists r\in R[\psi;j-1]$ and\\
		\> (3) $P(l+i-1,c+1)=+\partial\psi$, and\\
     (3) \= $\forall s \in R[\non \varphi]$ either\+\\
 	     (1) \= $s$ is $\partial$-sequence-discarded, or\\
	     (2) $\exists t\in R[\varphi]$ $\partial$-sequence-consumable, $t\succ s$, and\\
 	     (3) if $\exists w\in R[\non\varphi]$ $\partial$-sequence-applicable, $t\succ w$, then\+\\
             (1) \=$\forall \alpha_j\in A(t)$, $\alpha_j$ is $\partial$-consumed, otherwise\\
            (2) $\forall \alpha_k\in A(r)$, $\alpha_k$ is $\partial$-consumed.
\end{tabbing} 

Negative proof tags are trivial to deduce, and therefore omitted.

\begin{example}\label{ex:SemilastIhope}
	Consider $D=(\set{\alpha,\beta},R, \succ=\emptyset)$, where 
\allowdisplaybreaks
\begin{align*}
    R=\{ r_{0}: \mathit{\alpha} \To \mathit{\varphi};\mathit{\chi};\mathit{\psi}, \qquad r_{1}: \mathit{\beta} \To \mathit{\non\chi}\}.
\enlargethispage{.5\baselineskip}
\end{align*}
At $P(3,3)$ we prove $+\partial\varphi$, while $P(4,3)=-\partial\chi$ ($r_1$
is sequence-applicable and $r_0$ is not stronger than $r_1$). Therefore, $r_0$
cannot prove $\psi$.
\end{example}


 \subsection*{Proof tags for sequences in the antecedent and multi-sets in the conclusion} 
 \label{sec:proof_tags_for_sequences_in_the_antecedent_and_multi_sets_in_the_conclusion}
 
We lastly take into account multi-sets in the conclusion. When considering a
`team defeater fight', two scenarios are possible. In the former, we
draw a conclusion only if there is a winning team defeater for each literal
in the conclusion. In the latter, we limit the comparison on the individual
literal (and thus the latter solution is less strict than the former).
 
Strict provability does not change with respect to the one described in the
previous section, and is therefore omitted.

\begin{tabbing}
 $+\partial$: \= If $P(l+i,c+1)=+\partial \varphi$, then\+\\
 (1) \=$+\Delta \varphi\in P(l,c)$ or\\
 (2) \= (1) $-\Delta\non \varphi \in P(l,c)$ and\+\\
     (2) $\exists r\in R_{sd}[\varphi,j]$ $\partial$-sequence-consumable and\\
     (3) \= $\forall s \in R[\non \psi]$ such that $\psi\in C(r)$ either\+\\
 	     (1) \= $s$ is $\partial$-sequence-discarded, or\\
	     (2) $\exists t\in R[\psi]$ $\partial$-sequence-consumable, $t\succ s$, and\\
 	     (3) if $\exists w\in R[\non\varphi]$ $\partial$-sequence-applicable, $t\succ w$, then\+\\
             (1) \=$\forall \alpha_j\in A(t)$, $\alpha_j$ is $\partial$-consumed, otherwise\\
            (2) $\forall \alpha_k\in A(r)$, $\alpha_k$ is $\partial$-consumed.
\end{tabbing}

Consider $D$ of Example~\ref{ex:SemilastIhope}, where this time $C(r)=\set{\varphi,\chi,\psi}$. There is no rule stronger than $r_1$, and thus no conclusion can be defeasibly proven.

\begin{tabbing}
 $+\partial$: \= If $P(l+i,c+1)=+\partial \varphi$, then\+\\
 (1) \=$+\Delta \varphi\in P(l,c)$ or\\
 (2) \= (1) $-\Delta\non \varphi \in P(l,c)$ and\+\\
     (2) $\exists r\in R_{sd}[\varphi,j]$ $\partial$-sequence-consumable and\\
     (3) \= $\forall s \in R[\non \varphi]$ either\+\\
 	     (1) \= $s$ is $\partial$-sequence-discarded, or\\
	     (2) $\exists t\in R[\varphi]$ $\partial$-sequence-consumable, $t\succ s$, and\\
 	     (3) if $\exists w\in R[\non\varphi]$ $\partial$-sequence-applicable, $t\succ w$, then\+\\
             (1) \=$\forall \alpha_j\in A(t)$, $\alpha_j$ is $\partial$-consumed, otherwise\\
            (2) $\forall \alpha_k\in A(r)$, $\alpha_k$ is $\partial$-consumed.
\end{tabbing} 

Consider $D$ of Example~\ref{ex:SemilastIhope}, again with $C(r)=\set{\varphi,\chi,\psi}$. This time $r_1$ can prevent only to prove $+\partial\chi$ (and, analogously, $r_0$ prevents to conclude $+\partial\non\chi$), and thus we prove $+\partial\varphi$ as well as $+\partial\psi$.


\section{Conclusions, related and further work}\label{sec:Conclusions}

We have dealt with the problem of manipulating resource
consumption in non-monotonic reasoning. 
The combination of linear and defeasible features in a logical system is a
complete novelty in the community of computational logic and
knowledge representation.

Variants of SDL have been investigated so far as a
means for devising business process traces
\cite{DBLP:conf/prima/OlivieriGSC13,DBLP:conf/prima/OlivieriCG15,DBLP:conf/edoc/GhooshchiBGOS17}. While the idea is closely related to outline in the
Introduction that a derivation corresponds to a trace in a process, the
approach based on variants of SDL are not able to handle loops and, in
general, repetitions of tasks. These aspects are elegantly captured by the
sub-structural aspects presented in the paper.

Studies on light linear logic versions, with specific aspects
of linearity related to resource consumption have been devised such as
\emph{light} and \emph{soft linear logic} \cite{Gaboardi200867,Girard1998175}.
Applications of linear logic to problems indirectly related to business
processes such as Petri Nets can be found in \cite{Kanovich199748} and in
\cite{Tanabe1997156,DBLP:journals/apal/EngbergW97}. However, such approaches
are not able to handle in a natural fashion the aspect of exceptions. The
representation of exception would require complex rules and
encyclopaedic knowledge of the scenarios described by the processes encoded
by rules/sequents.

The most important properties we aim at proving for a logical non-monotonic
system are \emph{consistency} and \emph{coherence}\footnote{ In DL, a theory $D$ is \emph{consistent} if, for no literal $\varphi$, $D\vdash+\#\varphi$ and $D\vdash+\# \sim \varphi$; $D$ is \emph{coherent} if, for no literal $\varphi$, $D\vdash+\# \varphi$ and $D\vdash-\# \varphi$, $\#\in\set{\Delta,\partial}$.}. Although we cannot report a formal proof of consistency and coherence for the
logical system so far, we have exhaustively determined the conditions for
interference of the sub-structural and the non-monotonic aspects of RSDL, and
the formalisation of these derives the aforementioned properties.

A logical system enjoys the \emph{finite model property} when for every set
of formulae, the associated meaning to each formula requires a finite set of
elements in the semantics for every model of that set. In the case of RSDL,
the semantics is determined by the derivations that are possible given a
theory.

As a consequence of the aforementioned notions we shall prove one property
that regards acyclic RSDL. The Atom Dependency Graph\footnote{In the Atom Dependency Graph the atomic propositions are the nodes, and there is a directed edge between nodes if there is a rule containing the source or its negation in the body, and the target or its negation in the head.} of a defeasible
theory has been defined in many different contexts, specifically in the
analysis of preferences, as in \cite{Governatori2010104}. Acyclic RSDLs are
theories in which no cycle appears in the Atom Dependency Graph. This means
that when a rule is used to produce a conclusion, the resources in 
the antecedent of the rules cannot be replenished, and we reach nodes, literals, 
that can be produced by the theory only if they are given (node, with no incoming
edges in the Atom Dependency Graph). 

\begin{theorem}
Acyclic finite RSDL theories enjoy the Finite Model Property.
\end{theorem}
\begin{proof}
If there are no cycles in the Atom Dependency Graph, every time a rule is used
to derive a positive conclusion (strict or defeasible), the number of available resources
decreases. Accordingly, the maximum number of literals that can appear 
in a proof $P$ is bounded and proportional to the number of literals occurring 
in the head of rules in a given theory. Consequently, every derivation is finite.
Hence, the theory has the Finite Model Property. 
\end{proof}

Note that the theory with $\alpha$ as a fact and the rule $\alpha\To\alpha$
can generate a derivation with infinitely many occurrences of $\alpha$. 
%
%
The acyclicity condition allows us to compute in finite time the extension 
of a theory. However, this is not the case for cyclic theory, where the 
computation is not guaranteed to terminate. Accordingly, we can state the
following result. 

\begin{theorem}
The problem of computing extensions of cyclic RSDL theories is semi-decidable.
\end{theorem}

For acyclic theories, we have the finite model property. Therefore, since
acyclic theories can be checked for model existence in finite time, when the
model does not exists, and by brute force methods (for instance by simply
computing all the possible sequences of any finite length) we can trivially
claim the following result.

\begin{theorem}
The problem of computing extensions of acyclic RSDL theories is decidable.
\end{theorem}

Generally speaking, the extension computation problem is likely to be
decidable for larger classes of pure acyclic theories. In particular, when
cycles are conservative (cycles that do not produce new instances of the same
literal), or when a cycles are limited in the extension by some rules that
consume the literals generated in the cycles themselves before such literals
could be used to allow other rules fire, it is possible to guarantee the
finite model property. This is a matter of future investigations.

DL has been introduced as a means for managing non-monotonic
aspects of logical conclusion/derivation mechanisms. DL is
efficient in terms of time and space, being the problem of computing the
extension of a defeasible theory linear in the number of literals in the
theory. This property, however, cannot be claimed for RSDL. In particular, we
can show that RSDL can be used to represent classical 3-SAT problem, and
therefore prove that the complexity of this problem cannot be polynomial on
deterministic machines.

The basic idea of reducing 3-SAT is as follows. A 3-SAT problem $P$ is a
clause representing a finite conjunction of triplets ($t_i$), each formed by
three literals ($t_i^1,t_i^2,t_i^3$), that we assume to be conjuncted in the
sub-clause, where we ask whether the clause is satisfiable, or not. 

We map each literal appearing in the clause in a positive literal
$\widehat{t_i^x}$ (with $x =1,2,3$), not appearing in the clause, and add one
positive literal $\widehat{t_i}$ for every triplet, again not appearing in
the triplets. Subsequently, we add one rule $\widehat{t_i} \To \widehat{t_i^x}$ for
each of the three values $x=1,2,3$ and three rules $\widehat{t_i^x} \To
t_i^x$ for each of the values $x=1,2,3$. Finally, we add one fact for every
literal $\widehat{t_i}$. Conclusively, we have mapped every triplet in six
RSDL rules. The resulting RSDL theory has a derivation containing at
least one literal for each clause if and only if the original problem
$P$ is a satisfiable clause. For example, consider the clause 
$(\alpha \vee \beta \vee \gamma) \wedge (\neg \alpha \vee \neg \beta \vee \delta)$.
Using $c_1$ and $c_2$ for the triplets, and $c_1^1,c_1^2,c_1^3,c_2^1,c_2^2,c_2^3$ for the elements in the triplets, the theory encoding the clause is:

\vspace{-12pt}
\hspace{-2em}
\begin{minipage}[t]{.2\columnwidth}
\footnotesize
\begin{gather*}
c_1\\
c_2 \\
c_1 \To c_1^1 
\end{gather*}
\end{minipage}
\begin{minipage}[t]{.2\columnwidth}
\footnotesize
\begin{gather*}
c_1 \To c_1^2 \\
c_1 \To c_1^3 \\
c_2 \To c_2^1
\end{gather*}
\end{minipage}
\begin{minipage}[t]{.2\columnwidth}
\footnotesize
\begin{gather*}
c_2 \To c_2^2 \\
c_2 \To c_2^3 \\
c_1^1 \To \alpha 
\end{gather*}
\end{minipage}
\begin{minipage}[t]{.2\columnwidth}
\footnotesize
\begin{gather*}
c_1^2 \To \beta \\
c_1^3 \To \gamma \\
c_2^1 \To \non \alpha
\end{gather*}
\end{minipage}
\begin{minipage}[t]{.2\columnwidth}
\footnotesize
\begin{gather*}
c_2^2 \To \non \beta \\
c_2^3 \To \delta \\
\end{gather*}
\end{minipage}
\vspace{3pt}

In this paper we presented the general idea of how to develop a logic combining 
features for sub-structural logic and defeasible reasoning, and we provided
some general results about meta-theoretic properties (e.g., decidability and 
computational complexity). More research is required to determine the
correct boundary between decidable and undecidable problems for these types of
hybrid combinations and to provide a full map of the computational complexity 
analysis of the various options. However, the outline we discussed in this section 
seems to indicate that this is not a straightforward task. In this paper,
we did not address the issue of how to model the motivational attitudes of the 
agents, we left the investigation of how to extend the framework to integrate 
with the framework of \cite{jaamas:bio,tplp:goal}. Related to this, we shall
look at the problem of Business Process Compliance, in order to determine
how to employ RSDL for marking up traces of processes corresponding to the execution 
of the a theory.


\bibliographystyle{splncs03}
\bibliography{thisbiblio}


\end{document}